%% file: ArXivPaperTemplate.tex
\documentclass{article}

\def \TRtitle{Multi-Task Learning with Group-Specific Feature Space Sharing}
\def \TRauthors{Niloofar Yousefi, Michael Georgiopoulos and Georgios C. Anagnostopoulos}
\def \TRemails{niloofaryousefi@knights.ucf.edu, michaelg@ucf.edu and georgio@fit.edu}
\def \TRaffiliations{NY, MG: EE \& CS Dept., University of Central Florida; GCA: ECE Dept., Florida Institute of Technology}
\def \TRkeywords{Multi-task Learning, Kernel Methods, Generalization Bound, Support Vector Machines}

\usepackage{./sty/options}
\usepackage{./sty/packages}
\usepackage{./sty/macros}

\begin{document}

\maketitle

\ifMakeReviewDraft
	\linenumbers
\fi

\input{Abstract}


\input{Introduction}

\input{LearningMultipleTaskswithGroupLassoPenalty}

\input{ADMM}

\input{GeneralizationBound}

\input{Experiments}

\input{Conclusions}

\input{Acknowledgments}

\bibliographystyle{plainurl}
\bibliography{References}

\appendix
\input{Appendix}

\end{document}

%% file: Abstract.tex

\begin{abstract}
When faced with learning a set of inter-related tasks from a limited amount of usable data, learning each task independently may lead to poor generalization performance. \ac{MTL} exploits the latent relations between tasks and overcomes data scarcity limitations by co-learning all these tasks simultaneously to offer improved performance. We propose a novel \acl{MT-MKL} framework based on \aclp{SVM} for binary classification tasks. By considering pair-wise task affinity in terms of similarity between a pair's respective feature spaces, the new framework, compared to other similar \ac{MTL} approaches, offers a high degree of flexibility in determining how similar feature spaces should be, as well as which pairs of tasks should share a common feature space in order to benefit overall performance. The associated optimization problem is solved via a block coordinate descent, which employs a consensus-form \acl{ADMM} algorithm to optimize the \acl{MKL} weights and, hence, to determine task affinities. Empirical evaluation on seven data sets exhibits a statistically significant improvement of our framework's results compared to the ones of several other \acl{CMTL} methods.  
\end{abstract}


%% file: Introduction.tex
\acresetall

\section{Introduction}
\label{sec:Introduction}
\ac{MTL} is a machine learning paradigm, where several related task are learnt simultaneously with the hope that, by sharing information among tasks, the generalization performance of each task will be improved. The underlying assumption behind this paradigm is that the tasks are related to each other. Thus, it is crucial how to capture task relatedness and incorporate it into an \ac{MTL} framework. Although, many different \ac{MTL} methods  \cite{caruana1997multitask,evgeniou2007multi,jalali2010dirty,gu2011joint,zhang2012convex,argyriou2013learning} have been proposed, which differ in how the relatedness across multiple tasks is modeled, they all utilize the parameter or structure sharing strategy to capture the task relatedness.

However, the previous methods are restricted in the sense that they assume all tasks are similarly related to each other and can equally contribute to the joint learning process. This assumption can be violated in many practical applications as ``outlier'' tasks often exist. In this case, the effect of ``negative transfer", \ie, sharing information between irrelevant tasks, can lead to a degraded generalization performance. 

To address this issue, several methods, along different directions, have been proposed to discover the inherent relationship among tasks. For example, some methods \cite {bakker2003task,xue2007multi,zhang2012convex,zhang2014regularization}, use a regularized probabilistic setting, where sharing among tasks is done based on a common prior. These approaches are usually computationally expensive. Another family of approaches, known as the \ac{CMTL}, assumes that tasks can be clustered into groups such that the tasks within each group are close to each other according to a notion of similarity. Based on the current literature, clustering strategies can be broadly classified into two categories: task-level \ac{CMTL} and feature-level \ac{CMTL}.

The first one, task-level \ac{CMTL}, assumes that the model parameters used by all tasks within a group are close to each other. For example, in \cite{argyriou2008convex,evgeniou2005learning,jacob2009clustered}, the weight vectors of the tasks belonging to the same group are assumed to be similar to each other. However, the major limitations for these methods are: (i) that such an assumption might be too risky, as similarity among models does not imply that meaningful sharing of information can occur between tasks, and (ii) for these methods, the group structure (number of groups or basis tasks) is required to be known a priori.  

The other strategy for task clustering, referred to as feature-level \ac{CMTL}, is based on the assumption that task relatedness can be modeled as learning shared features among the tasks within each group. For example, in \cite {kang2011learning} the tasks are clustered into different groups and it is assumed that tasks within the same group can jointly learn a shared feature representation. The resulting formulation leads to a non-convex objective, which is optimized using an alternating optimization algorithm converging to local optima, and suffers potentially from slow convergence. Another similar approach has been proposed in \cite{rohtua2015}, which assumes that tasks should be related in terms of feature subsets. This study also leads to a non-convex co-clustering structure that captures task-feature relationship. These methods are restricted in the sense that they assume that tasks from different groups have nothing in common with each other. However, this assumption is not always realistic, as tasks in disjoint groups might still be inter-related, albeit weekly. Hence, assigning tasks into different groups may not take full advantage of \ac{MTL}. Another feature-level clustering model has been proposed in \cite{zhong2012convex}, in which the cluster structure can vary from feature to feature. While, this model is more flexible compared to other \ac{CMTL} methods, it is, however, more complicated and also less general compared to our framework, as it tries to find a shared feature representation for tasks by decomposing each task parameter into two parts: one to capture the shared structure between tasks and another to capture the variations specific to each task. This model is further extended in \cite {Han2015}, where a multi-level structure has been introduced to learn task groups in the context of \ac{MTL}. Interestingly, it has been shown that there is an equivalent relationship between \ac{CMTL} and alternating structure optimization  \cite{zhou2011clustered}, wherein the basic idea is to identify a shared low-dimensional predictive structure for all tasks.

In this paper, we develop a new \ac{MTL} model capable of modeling a more general type of task relationship, where the tasks are implicitly grouped according to a notion of feature similarity. In our framework, the tasks are not forced to have a common feature space; instead, the data automatically suggests a flexible group structure, in which a common, similar or even distinct feature spaces can be determined between different pairs of tasks. Additionally, our \ac{MTL} framework is kernel-based and, thus, may take advantage of the non-linearity introduced by the feature mapping of the associated \ac{RKHS} $\mathcal{H}$. Also, to avoid a degradation in generalization performance due to choosing an inappropriate kernel function, our framework employs a \ac{MKL} strategy \cite{lanckriet2004learning}, hence, rendering it a \ac{MT-MKL} approach.

It is worth mentioning that a widely adopted practice for combining kernels is to place an $L_p$-norm constraint on the combination coefficients $\boldsymbol{\theta}= [\theta_{1},\ldots,\theta_{M}]$, which are learned during training. For example, a conically combination of task objectives with an $L_p$-norm feasible region is introduced in \cite{li2014pareto} and further extended in \cite{li2014conic}. Also, another method introduced in \cite{tang2009multiple} proposes a partially shared kernel function $k_t \triangleq \sum_{m=1}^{M} (\mu^m + \lambda_t^m) k_m$, along with $L_1$-norm constraints on $\boldsymbol{\mu}$ and $\boldsymbol{\lambda}$. The main advantage of such a method over the traditional \ac{MT-MKL} methods, which consider a common kernel function for all tasks (by letting $\lambda_t^m=0, \forall t,m$), is that it allows tasks to have their own task-specific feature spaces and, potentially, alleviate the effect of negative transfer. However, popular \ac{MKL} formulations in the context of \ac{MTL}, such as this one, are capable of modeling two types of tasks: those that share a global, common feature space and those that employ their own, task-specific feature space. In this work we propose a more flexible framework, which, in addition to allowing some tasks to use their own specific feature spaces (to avoid negative transfer learning), it permits forming arbitrary groups of tasks sharing the same, group-specific (instead of a single, global), common feature space, whenever warranted by the data. This is accomplished by considering a group lasso regularizer applied to the set of all pair-wise differences of task-specific \ac{MKL} weights. For no regularization penalty, each task is learned independently of each other and will utilize its own feature space. As the regularization penalty increases, pairs of \ac{MKL} weights are forced to equal each other leading the corresponding pairs of tasks to share a common feature space. We demonstrate that the resulting optimization problem can be solved by employing a 2-block coordinate descent approach, whose first block consists of the \ac{SVM} weights for each task and which can be optimized efficiently using existing solvers, while its second block comprises the \ac{MKL} weights from all tasks and is optimized via a consensus-form, \ac{ADMM}-based step. 
 
The rest of the paper is organized as follows: In \sref{sec:Learning Multiple Tasks with a Group Lasso Penalty} we describe our formulation for jointly learning the optimal feature spaces and the parameters of all the tasks. \sref{sec:Optimization Algorithm} provides an optimization technique to solve our non-smooth convex optimization problem derived in \sref {sec:Learning Multiple Tasks with a Group Lasso Penalty}. \sref {sec:GeneralizationBounds} presents a Rademacher complexity-based generalization bound for the hypothesis space corresponding to our model.  Experiments are provided in \sref {sec:Experiments}, which demonstrate the effectiveness of our proposed model compared to several \ac{MTL} methods. Finally, in \sref{sec:Conclusions} we conclude our work and briefly summarize our findings.
 
 \textbf{Notation}: In what follows, we use the following notational conventions: vectors and matrices are depicted in bold face. A prime $'$ denotes vector/matrix transposition. The ordering symbols $\succeq$ and $\preceq$ when applied to vectors stand for the corresponding component-wise relations. If $\mathbb{Z}_+$ is the set of postivie integers, for a given $S \in \mathbb{Z}_+$, we define $\mathbb{N}_S \triangleq \left\{  1, \ldots, S \right\} $. Additional notation is defined in the text as needed.

%% file: LearningMultipleTaskswithGroupLassoPenalty.tex

\section {Formulation}
\label{sec:Learning Multiple Tasks with a Group Lasso Penalty}
Assume $T$ supervised learning tasks, each with a training set $\left\{ \left( x_t^n, y_t^n \right) \right\}_{n=1}^{n_t}, t \in \mathbb{N}_T$, which is sampled from an unknown distribution $P_t(x, y)$ on $\mathcal{X} \times \left\{ -1, 1\right\}$. Here, $\mathcal{X}$ denotes the native space of samples for all tasks and $\pm 1$ are the associated labels. Without loss of generality, we will assume an equal number $n$ of training samples per task. The objective is to learn $T$ binary classification tasks using discriminative functions $f_t(\boldsymbol{x}) \triangleq \left\langle  \boldsymbol{w}_t, \boldsymbol{\phi}_t(\boldsymbol{x}) \right\rangle_{\mathcal{H}_{t,\boldsymbol{\theta}}}+ b_{t}$ for $t \in \mathbb{N}_T$, where $\boldsymbol{w}_t$ is the weight vector associated to task $t$.  Moreover, the feature space of task $t$ is served by $\mathcal{H}_{t,\boldsymbol{\theta}} = \bigoplus_{m=1}^M \sqrt{\theta_t^m} \mathcal{H}_m$ with induced feature mapping $\phi_t \triangleq [ \sqrt{\theta_t^1}{\phi_1}' \cdots $ $ \sqrt{\theta_t^M} {\phi_M}']'$ and endowed with the inner product $\left\langle \cdot, \cdot \right\rangle_{\mathcal{H}_{t,\boldsymbol{\theta}}} = \sum_{m=1}^M \theta_t^m \left\langle \cdot, \cdot \right\rangle_{\mathcal{H}_m}$. The reproducing kernel function for this feature space is given as $k_t(x_t^i,x_t^j) = \sum_{m=1}^M \theta_t^m k_m(x_t^i,x_t^j)$ for all $x_t^i, x_t^j \in \mathcal{X}$. In our framework, we attempt to learn the $\boldsymbol{w}_t$'s and $b_t$'s jointly with the $\boldsymbol{\theta}_t$'s via the following regularized risk minimization problem: 

\noindent
\begin{align}
	\label{pro:primalFormulationEquivalent}
	\min_{\boldsymbol{w} \in \varOmega\left( \boldsymbol{w}\right), \boldsymbol{\theta}\in \varOmega\left( \boldsymbol{\theta}\right), \boldsymbol{b}} \  \sum_{t=1}^{T} & \frac{\| \boldsymbol{w}_t \|^2} {2}  + C \sum_{t=1}^{T}\sum_{i=1}^{n} \left[1-\boldsymbol{y}_t^i f_{t}(x_{t}^{i})\right]_{+} + \lambda \sum_{t=1}^{T-1}\sum_{s>t}^{T} \left\| \boldsymbol{\theta}_t - \boldsymbol{\theta}_s \right\|_2 
	\nonumber\\
	\varOmega\left ( \boldsymbol{w}\right)  \triangleq &
	\lbrace \boldsymbol{w} =  \left(\boldsymbol{w}_1,\cdots,\boldsymbol{w}_T\right) : \boldsymbol{w}_t \in \mathcal{H}_{t,\boldsymbol{\theta}}, \boldsymbol{\theta} \in \varOmega\left( \boldsymbol{\theta}\right) \rbrace \
	\nonumber\\
	\varOmega\left ( \boldsymbol{\theta}\right) \triangleq &
	\lbrace \boldsymbol{\theta} =  \left( \boldsymbol{\theta}_t,\cdots,\boldsymbol{\theta}_T \right) : \boldsymbol{\theta}_t \succeq \mathbf{0},  \| \boldsymbol{\theta}_t \|_1 \leq 1, \forall t \in \mathbb{N}_T \rbrace 
\end{align}

\noindent
where $\boldsymbol{w}\triangleq \left( \boldsymbol{w}_t,\cdots,\boldsymbol{w}_T\right)$ and $\boldsymbol{\theta}\triangleq \left( \boldsymbol{\theta}_t,\cdots,\boldsymbol{\theta}_T\right)$, $\varOmega\left( \boldsymbol{w}\right)$ and $\varOmega\left( \boldsymbol{\theta}\right)$ are the corresponding feasible sets for $\boldsymbol{w}$ and $\boldsymbol{\theta}$ respectively, and $[u]_+=\max \left\{ u, 0 \right\}, \ u \in \mathbb{R}$ denotes
the hinge function. Finally, $C$ and $\lambda$ are non-negative regularization parameters.

The last term in \probref{pro:primalFormulationEquivalent} is the sum of pairwise differences between the tasks' feature weight vectors. For each pair of $\left( \boldsymbol{\theta}_t, \boldsymbol{\theta}_s\right) $, the pairwise penalty $\|\boldsymbol{\theta}_t-\boldsymbol{\theta}_s\|_2$ may favor a small number of non-identical $\boldsymbol{\theta}_t$. Therefore, it ensures that a flexible  (common, similar or distinct) feature space, will be selected between tasks $t$ and $s$. In this manner, a flexible group structure of shared features across multiple tasks can be achieved by this framework. It is also worth mentioning that two special cases are covered by the proposed model: (i) if $\lambda \rightarrow \infty $ ($\lambda$ is only required to be sufficiently large), for all task pairs $\left\| \boldsymbol{\theta}_t - \boldsymbol{\theta}_s \right\|_2 \rightarrow 0$ and, thus, all tasks share a single common feature space. (ii) As $\lambda \rightarrow 0$, the proposed model reduces to $T$ independent classification tasks.

 It is easy to verify that \probref{pro:primalFormulationEquivalent} is a convex minimization problem, which can be solved using a block coordinate descent method alternating between the minimization with respect to $\boldsymbol{\theta}$ and the $(\boldsymbol{w}, \boldsymbol{b})$ pair. Motivated by the non-smooth nature of the last regularization term, in \sref{sec:Optimization Algorithm} we develop a consensus version of the \ac{ADMM} to solve the minimization problem with respect to $\boldsymbol{\theta}$.
 
%

%% file: ADMM.tex
\section{The proposed Consensus Optimization Algorithm}
\label{sec:Optimization Algorithm}


\probref{pro:primalFormulationEquivalent} can be formulated as the following equivalent problem, which entails $T$ inter-related \ac{SVM} training problems:

\noindent
\begin{align}
	\label{pro:primalFormulation-W}
	\min_{\boldsymbol{\theta}, \boldsymbol{w}, \boldsymbol{b}, \boldsymbol{\xi} }  \sum_{t=1}^{T} & \sum_{m=1}^M  \frac{\| \boldsymbol{w}_t^m \|^2_{\mathcal{H}_m}}{2 \theta_t^m}  +  C \sum_{t=1}^{T}\sum_{i=1}^{n} \xi_t^i + \lambda \sum_{t=1}^{T-1}\sum_{s>t}^{T} \left\| \boldsymbol{\theta}_t - \boldsymbol{\theta}_s \right\|_2 
	\nonumber\\
	\textit{s.t.}\; & y_t^i  \left( \left\langle w_t, \phi(x_t^i)\right\rangle_{\mathcal{H}_t} + b_t \right)  \geq 1-\xi_t^i ,\; \xi_t^i\geq 0,\; \forall\; t \in \mathbb{N}_T, i \in \mathbb{N}_n
	\nonumber\\
	& \boldsymbol{\theta}_t \succeq \mathbf{0},  \| \boldsymbol{\theta}_t \|_1 \leq 1, \forall\; t \in \mathbb{N}_T
\end{align}

It can be shown that the primal-dual form of \probref{pro:primalFormulation-W} with respect to $\boldsymbol{\theta}$ and $\{ \boldsymbol{w}, \boldsymbol{b}, \boldsymbol{\xi} \}$  is given by 
\begin{align}
	\label{pro:PrimalDualFormulation}
	 \min_{{\boldsymbol{\theta}}_t \in \varOmega (\boldsymbol{\theta})} \max_{\boldsymbol{\alpha}_t \in \varOmega (\boldsymbol{\alpha})}  & \sum_{t=1}^{T} \boldsymbol{\alpha}^{'}_t \mathbf{1}_n  -  \dfrac{1}{2}  \sum_{t=1}^{T} \sum_{m=1}^M \theta_t^m  (\boldsymbol{\alpha}^{'}_t Y_t  K_t^m  Y_t \boldsymbol{\alpha}_t)+ \lambda \sum_{t=1}^{T-1}\sum_{s>t}^{T} \left\| \boldsymbol{\theta}_t - \boldsymbol{\theta}_s \right\|_2 
	\nonumber\\
	 \varOmega \left( \boldsymbol{\alpha} \right) \triangleq &
		\lbrace \boldsymbol{\alpha} =  \left( \boldsymbol{\alpha}_t,\cdots,\boldsymbol{\alpha}_T \right) :  0 \preceq \boldsymbol{\alpha}_t \preceq C \mathbf{1}_n,\; \boldsymbol{\alpha}_t^{'} \boldsymbol{y}_t= 0 ,\; \forall\; t \in \mathbb{N}_T \rbrace 
	\nonumber\\
	 \varOmega\left ( \boldsymbol{\theta}\right) \triangleq &
		\lbrace \boldsymbol{\theta} =  \left( \boldsymbol{\theta}_t,\cdots,\boldsymbol{\theta}_T \right) : \boldsymbol{\theta}_t \succeq \mathbf{0},  \| \boldsymbol{\theta}_t \|_1 \leq 1, \forall\; t \in \mathbb{N}_T \rbrace 
\end{align}
\noindent
where $\mathbf{1}_n$ is a vector containing $n$ $1$'s, $\boldsymbol{Y}_t \triangleq diag(\boldsymbol{y}_t)$, $\boldsymbol{K}_t^m \in \mathbb{R}^{n\times n}$ is the kernel matrix, whose $(i,j)$ entry is given as $k_m (x_t^i,x_t^j)$, $\boldsymbol{\theta}_t \triangleq [\theta_t^1,\ldots,\theta_t^M]'$, and $\boldsymbol{\alpha}_t$ is the Lagrangian dual variable for the minimization problem \wrt $\{ \boldsymbol{w}_t, b_t, \boldsymbol{\xi}_t \}$. 

 It is not hard to verify that the optimal objective value of the dual problem is equal to the optimal objective value of the primal one, as the strong duality holds for the primal-dual optimization problems
 \wrt $\{ \boldsymbol{w}, \boldsymbol{b}, \boldsymbol{\xi} \}$ and $\boldsymbol{\alpha}$ respectively. Therefore, a block coordinate descent framework\footnote{A \texttt{MATLAB\textsuperscript{\circledR}} implementation of our framework is available at\\ \href{https://github.com/niloofaryousefi/ECML2015}{https://github.com/niloofaryousefi/ECML2015}} can be applied to decompose \probref{pro:PrimalDualFormulation} into two subproblems. The first subproblem, which is the maximization problem with respect to $\boldsymbol{\alpha}$, can be efficiently  solved via \texttt{LIBSVM} \cite{LIBSVM:CC01a}, and the second subproblem, which is the minimization problem with respect to $\boldsymbol{\theta}$, takes the form
\begin{align}
	\label{pro:PrimalFormulation-theta}
	  & \min_{{\boldsymbol{\theta}}_t}  \lambda \sum_{t=1}^{T-1}\sum_{s>t}^{T} \left\| \boldsymbol{\theta}_t - \boldsymbol{\theta}_s \right\|_2 + \sum_{t=1}^{T}  \boldsymbol{\theta}_t^{'} \boldsymbol{q}_t 
	  \nonumber\\
	&  \textit{s.t.}\;  \boldsymbol{\theta}_t \succeq \mathbf{0},  \| \boldsymbol{\theta}_t \|_1 \leq 1,\; \forall\; t \in \mathbb{N}_T
\end{align}
\noindent
where we defined $q_t^m \triangleq - \frac{1}{2}  \boldsymbol{\alpha}^{'}_t Y_t  K_t^m  Y_t \boldsymbol{\alpha}_t$ and $\boldsymbol{q}_t \triangleq  [q_t^1,\ldots,q_t^M]'$. Due to the non-smooth nature of \probref{pro:PrimalFormulation-theta}, we derive a consensus \ac{ADMM}-based optimization algorithm to solve it efficiently. Based on the exposition provided in Sections 5 and 7 of \cite{Boyd:2011}, it is straightforward to verify that \probref {pro:PrimalFormulation-theta} can be written in \ac{ADMM} form as 
\begin{align}
 \label{prob:consensus-theta}
 \min_{\boldsymbol{s}, \boldsymbol{\theta}, \boldsymbol{z}}\; & \lambda \sum _{i=1}^N  h_i(\boldsymbol{s}_i)+ g(\boldsymbol{\theta})+ I_{\varOmega\left ( \boldsymbol{\theta}\right)}(\boldsymbol{z})
    \nonumber\\
    \textit{s.t.}\ & \boldsymbol{s}_i - \tilde{\boldsymbol{\theta}_i}= \mathbf{0},\; i \in \mathbb{N}_N 
    \nonumber\\
    & \boldsymbol{z} - \boldsymbol{\theta} = \mathbf{0}
 \end{align} 
\noindent
where $N \triangleq \frac{T(T-1)}{2} $, and the local variable $\boldsymbol{s}_i \in \mathbb{R}^{2M}$ consists of two vector variables $(\boldsymbol{s}_i)_{j}$ and $(\boldsymbol{s}_i)_{j'}$, where $(\boldsymbol{s}_i)_{j}=\boldsymbol {\theta}_{\mathcal{M}_(i,j)}$. Note that the index mapping $t=\mathcal{M}(i,j)$ maps the $j$\textsuperscript{th} component of the local variable $\boldsymbol{s}_i$ to the $t$\textsuperscript{th} component of the global variable $\boldsymbol{\theta}$. Also, $\tilde{\boldsymbol{\theta}}_i$ can be considered as the global variable's idea of what the local variable $\boldsymbol{s}_i$ should be. Moreover, for each $i$, the function $h_i(\boldsymbol{s}_i)$ is defined as $\left\| (\boldsymbol{s}_i)_{j} - (\boldsymbol{s}_i)_{j'} \right\|_2 $, and the objective term $g(\boldsymbol{\theta})$ is given as $ \sum_{t=1}^{T}  \boldsymbol{\theta}_t^{'} \boldsymbol{q}_t $. Finally, $I_{\varOmega\left ( \boldsymbol{\theta}\right)} (\boldsymbol{z})$ is the indicator function for the constraint set $\boldsymbol{\theta}$ (\ie, $I_{\varOmega\left ( \boldsymbol{\theta}\right)} (\boldsymbol{z})=0$ for $\boldsymbol{z} \in \varOmega\left ( \boldsymbol{\theta}\right)$, and $I_{\varOmega\left( \boldsymbol{\theta} \right)} (\boldsymbol{z})= \infty$ for $\boldsymbol{z} \notin \varOmega\left ( \boldsymbol{\theta}\right)$). 

The augmented Lagrangian (using scaled dual variables) for \probref {prob:consensus-theta} is
\begin{align}
   \label{eq:ConsensusAugmented-theta}
    L_\rho(\boldsymbol{s},\boldsymbol{\theta},\boldsymbol{z},\boldsymbol{u},\boldsymbol{v} )  = & \lambda \sum_{i=1}^N  h_i(\boldsymbol{s}_i)+g(\boldsymbol{\theta})+ I_{\varOmega\left ( \boldsymbol{\theta}\right)}(\boldsymbol{z})+(\rho/2)\sum_{i=1}^N \lVert \boldsymbol{s}_i-\tilde{\boldsymbol{\theta}}_i + \boldsymbol{u}_i \rVert ^2_2 
    \nonumber\\ 
    &+(\rho/2) \lVert \boldsymbol{z}-\boldsymbol{\theta} + \boldsymbol{v} \rVert ^2_2,
\end{align}
\noindent
where $\boldsymbol{u}_i$ and $\boldsymbol{v}$ are the dual variables for the constraints $\boldsymbol{s}_i = \tilde{\boldsymbol{\theta}}_i$ and $\boldsymbol{z} = \boldsymbol{\theta} $ respectively. Applying \ac {ADMM} on the Lagrangian function given in \eqref{eq:ConsensusAugmented-theta}, the following steps are carried out in the $k$\textsuperscript{th} iteration 
\begin{align}
   \label{eq:ConsensusADMMSteps-theta}
    &  \boldsymbol{s}_i^{k+1}=\arg\min_{\boldsymbol{s}_i} \{ \lambda h_i(\boldsymbol{s}_i)+(\rho/2) \lVert \boldsymbol{s}_i-\tilde{\boldsymbol{\theta}}_i^{k} + \boldsymbol{u}_i^{k} \rVert ^2_2 \}\\
    \label{eq:ConsensusADMMSteps-theta-2}
    & \boldsymbol{\theta}^{k+1}=\arg\min_{\boldsymbol{\theta}} \{ g(\boldsymbol{\theta})+(\rho/2)\sum_{i=1}^N \lVert \boldsymbol{s}_i^{k+1}-\tilde{\boldsymbol{\theta}}_i + \boldsymbol{u}_i^{k} \rVert ^2_2 + (\rho/2) \lVert \boldsymbol{z}^{k}-\boldsymbol{\theta} + \boldsymbol{v}^{k} \rVert ^2_2 \} \\
    \label{eq:ConsensusADMMSteps-theta-3}
    & \boldsymbol{z}^{k+1}=\arg\min_{\boldsymbol{z}} \{ I_{\varOmega\left ( \boldsymbol{\theta}\right)}(\boldsymbol{z})+(\rho/2) \lVert \boldsymbol{z}-\boldsymbol{\theta}^{k+1} + \boldsymbol{v}^{k} \rVert ^2_2 \} \\
    \label{prob:ConsensusADMMSteps-theta-4}
    & \boldsymbol{u}_i^{k+1}=\boldsymbol{u}_i^{k}+ \boldsymbol{s}_i^{k+1}-\tilde{\boldsymbol{\theta}}_i^{k+1}\\
    \label{prob:ConsensusADMMSteps-theta-5}
    &
    \boldsymbol{v}^{k+1}=\boldsymbol{v}^{k}+ \boldsymbol{z}^{k+1}-\boldsymbol{\theta}^{k+1}
\end{align}
\noindent
where, for each $i \in \mathbb{N}_N$, the $\boldsymbol{s}$- and $\boldsymbol{u}$-updates can be carried out independently and in parallel. It is also worth mentioning that the $\boldsymbol{s}$-update is a proximal operator evaluation for $\lVert.\rVert_2$ which can be simplified to
\begin{align}
   \label{prob:ConsensusADMMSteps-theta-xupdate}
    &  \boldsymbol{s}_i^{k+1} = \mathcal{S}_{\lambda/\rho} (\tilde{\boldsymbol{\theta}}_i^{k} + \boldsymbol{u}_i^{k}),\; \forall\; i \in \mathbb{N}_N
\end{align}
\noindent
where $\mathcal{S}_{\kappa}$ is the vector-valued soft thresholding (or shrinkage) operator and which is defined as 
\begin{align}
   \label{prob:theta-xupdate-softthreshholding}
    &  \mathcal{S}_{\kappa} (\boldsymbol{a}) \triangleq (1-\kappa / \lVert \boldsymbol {a} \rVert_2 )_+ \boldsymbol{a}, \quad \mathcal{S}_{\kappa} (0) \triangleq 0.
\end{align}
\noindent
Furthermore, as the objective term $g$ is separable in $\boldsymbol{\theta}_t$, the $\boldsymbol{\theta}$-update can be decomposed into $T$ independent minimization problems, for which a closed from solution exists
\begin{align}
   \label{prob:ConsensusADMMSteps-theta-thetaupdate}
    &  \boldsymbol{\boldsymbol{\theta}}_t^{k+1} = \dfrac{1}{T-1} \left[ \sum_{\mathcal{M}(i,j)=t} \left( (\boldsymbol{s}_i)_j^{k+1} + (\boldsymbol{u}_i)_j^{k} \right) + \left( \boldsymbol{z}_t^{k} + \boldsymbol{v}_t^{k} \right) - (1/\rho) \boldsymbol{q}_t \right] ,\; \forall\; t \in \mathbb{N}_T
\end{align}

\begin{algorithm}
   \caption{Algorithm for solving \probref{pro:PrimalDualFormulation}.} 
   \label{alg1}
\begin{algorithmic}[1]
\REQUIRE
  $\boldsymbol{X}_{1},\ldots,\boldsymbol{X}_{T}, \boldsymbol{Y}_{1},\ldots,\boldsymbol{Y}_{T}, C, \lambda$ \\
 \ENSURE
  $ \boldsymbol{\theta}_{1},\ldots,\boldsymbol{\theta}_{T},\boldsymbol{\alpha}_{1},\ldots,\boldsymbol{\alpha}_{T}$\\
 \STATE
 \textbf{Initialize:} $ \boldsymbol{\theta}_{1}^{(0)},\ldots,\boldsymbol{\theta}_{T}^{(0)}$, $r=1$\\
 \STATE
 \textbf{Calculate:} Base kernel matrices $K_t^m$ using $\boldsymbol{X}_{t}$'s for the $T$ tasks and the $M$ kernels.
\WHILE{not converged}
\STATE $\boldsymbol{\alpha}^{(r)} \leftarrow \arg \max_{\ \boldsymbol{\alpha} \in \varOmega (\boldsymbol{\alpha})}\  \sum_{t=1}^{T} \boldsymbol{\alpha}^{'}_t \boldsymbol{e}  -  \dfrac{1}{2}  \sum_{t=1}^{T} \sum_{m=1}^M (\theta_t^{m})^{(r-1)}  (\boldsymbol{\alpha}^{'}_t Y_t  K_t^m  Y_t \boldsymbol{\alpha}_t)$\\
\STATE $(q_t^m)^{(r)} \leftarrow - \frac{1}{2}  (\boldsymbol{\alpha}^{'}_t)^{(r)} Y_t  K_t^m  Y_t  (\boldsymbol{\alpha}_t)^{(r)}, \; \forall t,m$
\STATE $\boldsymbol{\theta}^{(r)} \leftarrow \arg\min_{\ \boldsymbol{\theta} \in \varOmega (\boldsymbol{\theta}) }\  \lambda \sum_{t=1}^{T-1}\sum_{s>t}^{T} \left\| \boldsymbol{\theta}_t - \boldsymbol{\theta}_s \right\|_2 + \sum_{t=1}^{T}  \boldsymbol{\theta}_t^{'} \boldsymbol{q}_t^{(r)} $ using Algorithm \ref{alg2}
\ENDWHILE
\STATE $\boldsymbol{\alpha}^{*} = \boldsymbol{\alpha}^{(r)}$
\STATE $\boldsymbol{\theta}^{*} = \boldsymbol{\theta}^{(r)}$
\end{algorithmic}
\end{algorithm} 
\setlength\floatsep{0.02\baselineskip plus 2pt minus 1pt}
In the third step of the \ac {ADMM}, we project $(\boldsymbol{\theta}^{k+1} - \boldsymbol{v}^{k})$ onto the constraint set $\varOmega\left ( \boldsymbol{\theta}\right)$. Note that, this set is 
separable in $\boldsymbol{\theta}$, so the projection step can also be performed independently and in parallel for each variable $\boldsymbol{z}_t$, \ie,
\begin{align}
   \label{prob:ConsensusADMMSteps-theta-zupdate}
    &  \boldsymbol{z}_t^{k+1} = \mathbf{\Pi}_{\varOmega\left ( \boldsymbol{\theta}\right)} (\boldsymbol{\theta}_t^{k+1} + \boldsymbol{v}_t^{k}), \; \forall\; t \in \mathbb{N}_T. 
\end{align}

The $\boldsymbol{z}_t$-update can also be seen as the problem of finding the intersection between two closed convex sets $ \varOmega_1 \left ( \boldsymbol{\theta}\right)= \{\boldsymbol{\theta}_t \succeq \mathbf{0},\; \forall\; t \in \mathbb{N}_T \}$ and $\varOmega_2 \left ( \boldsymbol{\theta}\right) =\{ \| \boldsymbol{\theta}_t \|_1 \leq 1,\; \forall\; t \in \mathbb{N}_T \}$, which can be handled using Dykstra's alternating projections method \cite{bauschke1994,dykstra1983} as follows
\begin{align}
   \label{eq:ConsensusADMMSteps-theta-zupdate-2}
    &  \boldsymbol{y}_t^{k+1} = \mathbf{\Pi}_{\varOmega_1 \left ( \boldsymbol{\theta}\right)} (\boldsymbol{\theta}_t^{k+1} + \boldsymbol{v}_t^{k} - \boldsymbol{\beta}_t^{k})= \frac{1}{2} \left[  \boldsymbol{\theta}_t^{k+1} + \boldsymbol{v}_t^{k} - \boldsymbol{\beta}_t^{k}  \right]_+,\; \forall\; t \in \mathbb{N}_T 
    \\
    \label{eq:ConsensusADMMSteps-theta-zupdate-2-2}
    &  \boldsymbol{z}_t^{k+1} = \mathbf{\Pi}_{\varOmega_2 \left ( \boldsymbol{\theta}\right)} (\boldsymbol{y}_t^{k+1} +  \boldsymbol{\beta}_t^{k})= \mathbf{P}_M (\boldsymbol{y}_t^{k+1} +  \boldsymbol{\beta}_t^{k}) + \frac{1}{M} \mathbf{1}_M,\; \forall\; t \in \mathbb{N}_T 
    \\
    \label{eq:ConsensusADMMSteps-theta-zupdate-2-3}
    & \boldsymbol{\beta}_t^{k+1}= \boldsymbol{\beta}_t^{k}+  \boldsymbol{y}_t^{k+1} - \boldsymbol{z}_t^{k+1},\; \forall\; t \in \mathbb{N}_T
\end{align}
\noindent
where $\mathbf{P}_M \triangleq  \left( \mathbf{I}_M-\frac{\mathbf{1}_M \mathbf{1}_M^{'}}{M} \right)$ is the centering matrix. Furthermore, the $\boldsymbol{y}_t$- and $\boldsymbol{z}_t$ updates are the Euclidean projections onto $\varOmega_1 \left ( \boldsymbol{\theta}\right)$ and $\varOmega_2 \left ( \boldsymbol{\theta}\right)$ respectively with dual variables $\boldsymbol{\beta}_t \in \mathbb{R}^{M\times 1},\;t=1,\ldots,T$. Finally, we update the dual variables $\boldsymbol{u}_i$ and $\boldsymbol{v}$ using the equations given in \eqref{prob:ConsensusADMMSteps-theta-4} and \eqref{prob:ConsensusADMMSteps-theta-5}.

\setlength\textfloatsep{1\baselineskip plus 3pt minus 2pt}
\begin{algorithm}
   \caption{Consensus \ac{ADMM} algorithm to solve optimization \probref{pro:PrimalFormulation-theta}}
   \label{alg2}
\begin{algorithmic}[1]
\REQUIRE
 $\boldsymbol{q}_{1}^{(r)},\ldots,\boldsymbol{q}_{T}^{(r)},  \rho$ \\
 \ENSURE
 $ \boldsymbol{\theta}_{1}^{(r)},\ldots,\boldsymbol{\theta}_{T}^{(r)}$\\
 \STATE
 \textbf{Initialize:}  $\hat{\boldsymbol{\theta}}_{1}^{(0)},\ldots,\hat{\boldsymbol{\theta}}_{T}^{(0)}, k=0$\\
\WHILE{not converged}
\FOR{$ i \in \mathbb{N}_N, t \in \mathbb{N}_T$}
\STATE $\boldsymbol{s}_i^{k+1} \leftarrow  \mathcal{S}_{\lambda/\rho} (\tilde{\boldsymbol{\theta}}_i^{k} + \boldsymbol{u}_i^{k})$\\
\STATE $\hat{\boldsymbol{\theta}}_{t}^{k+1} \leftarrow  \dfrac{1}{T-1} \left[ \sum_{\mathcal{M}(i,j)=t} \left( (\boldsymbol{s}_i)_j^{k+1} + (\boldsymbol{u}_i)_j^{k} \right) + \left( \boldsymbol{z}_t^{k} + \boldsymbol{v}_t^{k} \right) - (1/\rho) \boldsymbol{q}_t \right]$ \\
\STATE $\boldsymbol{y}_t^{k+1} \leftarrow  \frac{1}{2} \left[  \hat{\boldsymbol{\theta}}_t^{k+1} + \boldsymbol{v}_t^{k} - \boldsymbol{\beta}_t^{k}  \right]_+$\\
\STATE  $\boldsymbol{z}_t^{k+1} \leftarrow \mathbf{P}_M (\boldsymbol{y}_t^{k+1} +  \boldsymbol{\beta}_t^{k}) + \frac{1}{M}\mathbf{1}_M$\\
\STATE $\boldsymbol{\beta}_t^{k+1}  \leftarrow  \boldsymbol{\beta}_t^{k}+  \boldsymbol{y}_t^{k+1} - \boldsymbol{z}_t^{k+1}$ \\
\STATE $\boldsymbol{u}_i^{k+1}  \leftarrow \boldsymbol{u}_i^{k}+ \boldsymbol{s}_i^{k+1}-\tilde{\boldsymbol{\theta}}_i^{k+1}$ \\
\STATE $\boldsymbol{v}_{t}^{k+1} \leftarrow \boldsymbol{v}_{t}^{k}+ \boldsymbol{z}_{t}^{k+1}- \hat{\boldsymbol{\theta}}_{t}^{k+1}$
\ENDFOR   
\ENDWHILE
 \STATE $\boldsymbol{\theta}^{(r)} \leftarrow \hat{\boldsymbol{\theta}} ^{(k+1)}$      
\end{algorithmic}
\end{algorithm}
\noindent

\subsection{Convergence Analysis and Stopping Criteria}

 Convergence of \aref{alg2} can be derived based on two mild assumptions similar to the standard convergence theory of the \ac {ADMM} method discussed in \cite{Boyd:2011}; (i) the objective functions $h(\boldsymbol{s})=\sum_{i=1}^{N} \left\| (\boldsymbol{s}_i)_{j} - (\boldsymbol{s}_i)_{j'} \right\|_2$ and $g(\boldsymbol{\theta})=\sum_{t=1}^{T}  \boldsymbol{\theta}_t^{'} \boldsymbol{q}_t$ are closed, proper and convex, which implies that the subproblems arising in the $\boldsymbol{s}$-update \eqref{eq:ConsensusADMMSteps-theta} and $\boldsymbol{\theta}$-update \eqref{eq:ConsensusADMMSteps-theta-2} are solvable, and (ii) the augmented Lagrangian \eqref{eq:ConsensusAugmented-theta} for $\rho=0$ has a saddle point. Under these two assumptions, it can be shown that our \ac {ADMM}-based algorithm satisfies the following
 \begin{itemize}
 \item
Convergence of residuals : $ {\boldsymbol{s}_i}^{k} - \tilde { \boldsymbol{\theta}}_i^{k} \rightarrow \mathbf{0},\; \forall\; i \in \mathbb{N}_N $, and $\boldsymbol{z}^{k} - \boldsymbol{\theta}^{k} \rightarrow \mathbf{0}$ as $k \rightarrow \infty$.
 \item
Convergence of dual variables: $\boldsymbol{u}_i^{k} \rightarrow \boldsymbol{u}_i^{*}, \forall i \in \mathbb{N}_N$, and $\boldsymbol{v}^{k} \rightarrow \boldsymbol{v}^*$ as $k \rightarrow \infty$, where $\boldsymbol{u}^{*}$ and $\boldsymbol{v}^{*}$ are the dual optimal points.
 \item
Convergence of the objective : $h(\boldsymbol{s}^{k})+g(\boldsymbol{z}^{k}) \rightarrow p^{*}$ as $k \rightarrow \infty$, which means the objective function \eqref{pro:PrimalFormulation-theta} converges to its optimal value as the algorithm proceeds.
 \end{itemize}

Also, the algorithm is terminated, when the primal and dual residuals satisfy the following stopping criteria
\begin{align}
\label{eq:stopping criteria}
 &\|e_{p_{1}}^{k} \|_2 \leq \epsilon_{1}^{pri},\quad \ \, \|e_{p_{2}}^{k} \|_2 \leq \epsilon_{2}^{pri}, \quad \ \: \|e_{p_{3}}^{k} \|_2 \leq \epsilon_{3}^{pri}
\nonumber\\
 & \|e_{d_{1}}^{k} \|_2 \leq \epsilon_{1}^{dual},\quad  \|e_{d_{2}}^{k} \|_2 \leq \epsilon_{2}^{dual}, \quad \|e_{d_{3}}^{k} \|_2 \leq \epsilon_{3}^{dual}
\end{align}
\noindent
where the primal residuals of the $k$\textsuperscript{th} iteration are given as $e_{p_{1}}^{k}= \boldsymbol{s}^{k} -  \boldsymbol{\theta}^{k}$, $e_{p_{2}}^{k}=\boldsymbol{z}^{k} - \boldsymbol{\theta}^{k} $ and $e_{p_{3}}^{k}= \boldsymbol{y}^{k} -  \boldsymbol{z}^{k}$. Similarly  $e_{d_{1}}^{k} = \rho (\boldsymbol{\theta}^{k+1} - \boldsymbol{\theta}^{k})$, $e_{d_{2}}^{k} = \rho (\boldsymbol{z}^{k} - \boldsymbol{z}^{k+1}) $ and $e_{d_{3}}^{k} = \rho (\boldsymbol{y}^{k} - \boldsymbol{y}^{k+1})$are dual residuals at iteration $k$. Also, the tolerances $\epsilon^{pri} > 0$, and $\epsilon^{dual} > 0$ can be chosen appropriately using the method described in Chapter 3 of \cite{Boyd:2011}. 



\subsection{Computational Complexity}

 \aref{alg1} needs to compute and cache $TM$ kernel matrices; however, they are computed only once in $\mathcal{O}(TMn^2)$ time. Also, as long as the number of tasks $T$ is not excessive, all the matrices can be computed and stored on a single machine, since (i) the number $M$ of kernels, is typically chosen small (\eg, we chose $M=10$), and (ii) the number $n$ of training samples per task is not usually large; if it were large, \ac{MTL} would probably not be able to offer any advantages over training each task independently. For each iteration of \aref{alg1}, $T$ independent \ac{SVM} problems are solved at a time cost of $\mathcal{O}(n^3)$ per task. Therefore, if \aref{alg2} converges in $K$ iterations, the runtime complexity of \aref{alg1} becomes $\mathcal{O}(Tn^3+KMT^2)$ per iteration. Note, though, that $K$ is not usually more than a few tens of iterations \cite{Boyd:2011}.

On the other hand, if the number of tasks $T$ is large, the nature of our problem allows our algorithm to be implemented in parallel. The $\boldsymbol{\alpha}$-update can be handled as $T$ independent optimization problems, which can be easily distributed to $T$ subsystems. Each subsystem $N$ needs to compute once and cache $M$ kernel matrices for each task. Then, for each iteration, one \ac{SVM} problem is required to be solved by each subsystem, which takes $\mathcal{O}(n^3)$ time. Moreover, our \ac {ADMM}-based algorithm updating the $\boldsymbol{\theta}$ parameters can also be implemented in parallel over $i \in \mathbb{N}_N $. Assuming that exchanging data and updates between subsystems consumes negligible time, the \ac{ADMM} only requires $\mathcal{O}(KM)$ time. Therefore, taking advantage of a distributed implementation, the complexity of \aref{alg1} is only $\mathcal{O}(n^3+KM)$ per iteration.

%% file: GeneralizationBound.tex
\section{Generalization Bound based on Rademacher Complexity}
\label{sec:GeneralizationBounds}

In this section, we provide a Rademacher complexity-based generalization bound for the \ac{HS} considered in \probref{pro:primalFormulationEquivalent}, 
which can be identified with the help of the following Proposition \footnote{Note that Proposition \ref{Propos:Proposition1} here utilizes the first part of Proposition 12 in \cite{kloft2011lp} and does not require the strong duality assumption, which is necessary for the second part of Proposition 12 in \cite{kloft2011lp}.}. 

\begin{prop}
\label{Propos:Proposition1} (Proposition 12 in \cite{kloft2011lp}, part (a))
Let $\mathcal{C} \subseteq \mathcal{X}$ and let $f, g: \mathcal{C} \mapsto \mathbb{R}$ be two functions. For any $\nu > 0$, there must exist a $\eta > 0$, such that the optimal solution of \eref{Klof-first part1} is also optimal in \eref{Klof-first part2}  
\begin{align}
\label{Klof-first part1}
\min_{x \in \mathcal{C}} f(x)+\nu g(x)
\\
\label{Klof-first part2}
\min_{x \in \mathcal{C}, g(x) \leq \eta} f(x)
\end{align}
\end{prop}

Using Proposition \ref{Propos:Proposition1}, one can show that \probref {pro:primalFormulationEquivalent} is equivalent to the following problem
\noindent
\begin{align}
	\label{pro:primalFormulationEquivalent-2}
	 & \min_{\boldsymbol{w}\in\varOmega^{'}\left( \boldsymbol{w}\right)} \ C \sum_{t=1}^{T}\sum_{i=1}^{n} l\left( \boldsymbol{w}_t,\phi_t\left( x_t^i \right),y_t^i\right)
	\nonumber\\
    \varOmega^{'} \left( \boldsymbol{w}\right)  \triangleq  &
	\lbrace \boldsymbol{w} =  \left( \boldsymbol{w}_1,\cdots,\boldsymbol{w}_T\right) : \boldsymbol{w}_t \in \mathcal{H}_{t,\boldsymbol{\theta}}, \boldsymbol{\theta} \in \varOmega^{'}\left( \boldsymbol{\theta}\right),\;  \| \boldsymbol{w}_t \|^2 \leq R_{t},\; t \in \mathbb{N}_T \rbrace 
\end{align}
\noindent
where
\begin{align}
    \varOmega^{'}\left ( \boldsymbol{\theta}\right)  \triangleq \varOmega\left ( \boldsymbol{\theta}\right) \cap
	\left\lbrace  \boldsymbol{\theta}  =  \left( \boldsymbol{\theta}_t ,\cdots,\boldsymbol{\theta}_T  \right) : \sum_{t=1}^{T-1}\sum_{s>t}^{T} \left\| \boldsymbol{\theta}_t - \boldsymbol{\theta}_s \right\|_2 \leq \gamma  \right\rbrace 
\nonumber
\end{align}

The goal here is to choose the $\boldsymbol{w}$ and $\boldsymbol{\theta}$ from their relevant feasible sets, such that the objective function of \eqref{pro:primalFormulationEquivalent-2} is minimized. Therefore, the relevant hypothesis space for \probref{pro:primalFormulationEquivalent-2} becomes 
\begin{align}
     \label{eq:hypothesis}
     \mathcal{F}  \triangleq 
     \left\lbrace   x \mapsto \left[  \langle \boldsymbol{w}_1, \boldsymbol{\phi}_1 \rangle,\ldots,\langle \boldsymbol{w}_T, \boldsymbol{\phi}_T \rangle  \right] ^{'}:  \forall t  \boldsymbol{w}_t \in \mathcal{H}_{t,\boldsymbol{\theta}}, \| \boldsymbol{w}_t \|^2 \leq R_{t},  \boldsymbol{\theta} \in \varOmega^{'} ({\boldsymbol{\theta}}) \right\rbrace 
\end{align}\

Note that finding the \ac {ERC} of $\mathcal{F}$ is complicated due to the non-smooth nature of the constraint $\sum_{t=1}^{T-1}\sum_{s>t}^{T} \left\| \boldsymbol{\theta}_t - \boldsymbol{\theta}_s \right\|_2 \leq \gamma$. Instead, we will find the \ac {ERC} of the \ac {HS} $\mathcal{H}$ defined in \eqref{eq:hypothesisH}; notice that $\mathcal{F} \subseteq \mathcal{H}$.    

\begin{align}
     \label{eq:hypothesisH}
     \mathcal{H}  \triangleq 
     \left\lbrace   x \mapsto \left[  \langle \boldsymbol{w}_1, \boldsymbol{\phi}_1 \rangle,\ldots,\langle \boldsymbol{w}_T, \boldsymbol{\phi}_T \rangle  \right] ^{'}: \forall t  \boldsymbol{w}_t \in \mathcal{H}_{t,\boldsymbol{\theta}},  \| \boldsymbol{w}_t \|^2 \leq R_{t},  \boldsymbol{\theta} \in \varOmega^{''} ({\boldsymbol{\theta}}) \right\rbrace 
\end{align}
\noindent
where 
\begin{align}
\label{ThetaSpace3}
    \varOmega^{''}\left ( \boldsymbol{\theta}\right)  \triangleq \varOmega\left ( \boldsymbol{\theta}\right) \cap
	\left\lbrace   \boldsymbol{\theta}  =  \left( \boldsymbol{\theta}_t ,\cdots,\boldsymbol{\theta}_T  \right) : \sum_{t=1}^{T-1}\sum_{s>t}^{T} \left\| \boldsymbol{\theta}_t - \boldsymbol{\theta}_s \right\|_2^2 \leq \gamma^2  \right\rbrace 
\end{align}

Using the first part of Theorem (12) in \cite{bartlett2003rademacher}, it can be shown that the \ac {ERC} of $\mathcal{H}$ upper bounds the \ac {ERC} of function class $\mathcal{F}$. Thus, the bound derived for $\mathcal{H}$ is also valid for $\mathcal{F}$.
The following theorem provides the generalization bound for $\mathcal{H}$.

\begin{thrm}
\label{MainTheorem}
Let $\mathcal{H}$ defined in \eqref{eq:hypothesisH} be the multi-task \ac {HS} for a class of functions $\boldsymbol{f}=(f_1,\ldots,f_T) : \mathcal{X}\mapsto \mathbb{R}^T$. Then for all $f \in \mathcal{H}$, for $\delta > 0$ and for fixed $\rho > 0$, with probability at least $ 1- \delta $ it holds that
\begin{align}
\label{eq:REC-Bound-Multi-task-2}
 R(\boldsymbol{f}) \leq \hat{R}_{\rho}(\boldsymbol{f}) + \frac{2}{\rho} \hat{\mathfrak{R}}_{S} (\mathcal{H}) + 3 \sqrt{\dfrac{\log\frac{1}{\delta}}{2Tn}}
\end{align}
\noindent

where

\begin{align}
\label{Rademacher-Upperbound-final-1}
\hat{\mathfrak{R}}_{S} \left( \mathcal{H} \right) \leq \hat{\mathfrak{R}}_{ub} \left( \mathcal{H} \right) = \sqrt{ \frac { \sqrt{3} \gamma R M} {n T} } 
\end{align}

\noindent
where $\hat{\mathfrak{R}}_S (\mathcal{H})$, the \ac {ERC} of $\mathcal{H}$, is given as 
\begin{align}
 \label{eq:RademacherMulti-task}
   \hat{\mathfrak{R}}_{S} (\mathcal{H}) = \frac{1}{nT} \mathrm{E}_{\sigma} \left\lbrace  \sup_{\boldsymbol{f}=(f_1,\ldots,f_T) \in \mathcal{F}} \sum_{t=1}^{T} \sum_{i=1}^{n} \sigma_{t}^{i} f_{t}(x_{t}^{i}) \Biggr \vert \left\{ x_t^i \right\}_{t \in \mathbb{N}_T, i \in \mathbb{N}_n}  \right\rbrace 
 \end{align}
\noindent
the $\rho$-empirical large margin error $\hat{R}_{\rho} (\boldsymbol{f})$, for the training sample $S = \left\{ \left( x_t^i, y_t^i \right) \right\}_{i,t=1}^{n,T}$ is defined as
\noindent
\begin{align}
\hat{R}_{\rho} (\boldsymbol{f})= \frac{1}{nT} \sum_{t=1}^{T} \sum_{i=1}^{n} \min \left( 1, [1-y_{t}^{i} f_{t}(x_{t}^{i})/\rho]_+   \right) 
\nonumber
\end{align}
\noindent
Also, $R(f)= \Pr [y f(x) < 0] $ is the expected risk \wrt \,  0-1 loss, $n$ is the number of training samples for each task, $T$ is the number of tasks to be trained, and $M$ is the number of kernel functions utilized for \ac {MKL}.
\end{thrm}
The proof of this theorem is omitted due to space constraints. Based on \thrmref{MainTheorem}, the second term in \eref{eq:REC-Bound-Multi-task-2}, the upper bound for \ac {ERC} of $\mathcal{H}$, decreases as the number of tasks increases. Therefore, it is reasonable to expect that the generalization performance to improve, when the number $T$ of tasks or the number $n$ of training samples increase. Also, due to the formulation's group lasso ($L_1/L_2$-norm) regularizer on the pair-wise \ac{MKL} weight differences, the \ac{ERC} in \eref{Rademacher-Upperbound-final-1} depends on $M$ as $\mathcal{O}{\sqrt{M}}$. It is worth mentioning, that, while this could be improved to $\mathcal{O}{\sqrt{\log M}}$ as in \cite{cortes2010generalization}, if one considers instead a $L_p/L_q$-norm regularizer, we won't pursue this avenue here. Let us finally note, that \eref{eq:REC-Bound-Multi-task-2} allows one to construct data-dependent confidence intervals for the true, pooled (averaged over tasks) misclassification rate of the \ac{MTL} problem under consideration.

%% file: Experiments.tex
\section{Experiments}
\label{sec:Experiments}
In this section, we demonstrate the merit of the proposed model via a series of comparative experiments. For reference, we consider two baseline methods referred to as \textbf{STL} and \textbf{MTL}, which present the two extreme cases discussed in \sref{sec:Learning Multiple Tasks with a Group Lasso Penalty}. We also compare our method with five state-of-the-art methods which, like ours, fall under the \ac{CMTL} family of approaches. These methods are briefly described below.
\begin{itemize}
\item
\textbf{STL}: single-task learning approach used as a baseline, according to which each task is individually trained via a traditional single-task \ac {MKL} strategy.
\item
\textbf{MTL}: a typical \ac {MTL} approach, for which all tasks share a common feature space. An \ac{SVM}-based formulation with multiple kernel functions was utilized and the common \ac{MKL} parameters for all tasks were learned during training.
\item
\textbf{CMTL} \cite{jacob2009clustered}: in this work, the tasks are grouped into disjoint clusters, such that the model parameters of the tasks belonging to the same group are close to each other.
\item
\textbf{Whom} \cite{kang2011learning}: clusters the task, into disjoint groups and assumes that tasks of the same group can jointly learn a shared feature representation.
\item
\textbf{FlexClus} \cite{zhong2012convex}: a flexible clustering structure of tasks is assumed, which can vary from feature to feature.
\item
\textbf{CoClus} \cite{rohtua2015}: a co-clustering structure is assumed aiming to capture both the feature and task relationship between tasks.
\item
\textbf{MeTaG} \cite{Han2015}: a multi-level grouping structure is constructed by decomposing the matrix of tasks' parameters into a sum of components, each of which corresponds to one level and is regularized with a $L_2$-norm on the pairwise difference between parameters of all the tasks.  
\end{itemize}

\subsection{Experimental Settings}
\label{sec:exp_setting}
For all experiments, all kernel-based methods (including \textbf{STL}, \textbf{MTL} and our method) utilized $1$ Linear, $1$ Polynomial with degree $2$, and $8$ Gaussian kernels with spread parameters $\left\{2^0,\ldots,2^7\right\}$ for \ac{MKL}. All kernel functions were normalized as $k(\boldsymbol{x}, \boldsymbol{y}) \leftarrow k(\boldsymbol{x}, \boldsymbol{y})/\sqrt{k(\boldsymbol{x}, \boldsymbol{x}) k(\boldsymbol{y}, \boldsymbol{y})}$. Moreover, for \textbf{CMTL}, \textbf{Whom} and \textbf{CoClus} methods, which require the number of task clusters to be pre-specified, cross-validation over the set $\left\{1,\ldots, T/2 \right\}$ was used to select the optimal number of clusters. Also, the regularization parameters of all methods were chosen via cross-validation over the set $\left\{2^{-10},\ldots,2^{10}\right\}$.

\subsection{Experimental Results}

We assess the performance of our proposed method compared to the other methods on $7$ widely-used data sets including $3$ real-world data sets: Wall-Following Robot Navigation (\textit{Robot}), Statlog Vehicle Silhouettes (\textit{Vehicle}) and Statlog Image Segmentation (\textit{Image}) from the UCI repository \cite{Frank2010}, $2$ handwritten digit data sets, namely MNIST Handwritten Digit (\textit{MNIST}) and Pen-Based Recognition of Handwritten Digits (\textit{Pen}), as well as \textit{Letter} and \textit{Landmine}. 

The data sets from the UCI repository correspond to three multi-class problems. In the \textit{Robot} data set, each sample is labeled as: ``Move-Forward”, ``SlightRight-Turn", ``Sharp-Right-Turn" and ``Slight-Left-Turn". These classes are designed to navigate a robot through a room following the wall in a clockwise direction. The \textit{Vehicle} data set describes four different types of vehicles as ``4 Opel", ``SAAB", ``Bus" and ``Van".  On the other hand, the instances of the \textit{Image} data set were drawn randomly from a database of 7 outdoor images which are labeled as ``Sky", ``Foliage", ``Cement", ``Window", ``Path" and ``Grass". 

Also, two multi-class handwritten digit data sets, namely \textit{MNIST} and \textit{Pen}, consist of samples of handwritten digits from 0 to 9. Each example is labeled as one of ten classes. A one-versus-one strategy was adopted to cast all multi-class learning problems into \ac{MTL} problems, and  the average classification accuracy across tasks was calculated for each data set. Moreover, an equal number of samples from each class was chosen for training for all five multi-class problems.
 
We also compare our method on two widely-used multi-task data sets, namely the \textit{Letter} and \textit{Landmine} data sets. The former one is a collection of handwritten words collected by Rob Kassel of MIT's spoken Language System Group, and involves eight tasks: `C' vs. `E', `G' vs. `Y', `M' vs. `N', `A' vs. `G', `I' vs. `J', `A' vs. `O', `F' vs. `T' and `H' vs. `N'. Each letter is represented by a 8 by 16 pixel image, which forms a 128 dimensional feature vector per sample.  We randomly chose 200 samples for each letter. An exception is letter ‘J’, for which only 189 samples were available. The \textit{Landmine} data set consists of 29 binary classification tasks collected from various landmine fields.  The objective is to recognize whether there is a landmine or not based on a region's characteristics, which are described by four moment-based features, three correlation-based features, one energy ratio feature, and one spatial variance feature. 

In all our experiments, for all methods, we considered training set sizes of $10\%$, $20\%$ and $50\%$ of the original data set to investigate the influence of the data set size on generalization performance. An exception was the \textit{Landmine} data set, for which we used $20\%$ and $50\%$ of the data set for training purposes due to its small size. The rest of data were split into equal sizes for validation and testing. 

\begin{table}[htpb]
 \begin{center}
 \caption{Experimental comparison between our method and seven benchmark methods}
 \label{multi-task}
 \tabcolsep=0.12cm
\resizebox{\columnwidth}{!}{%
 \begin{tabular}{l c c c c c c c c}
  \toprule

  10\%&	STL$^{(7)}$& 	MTL$^{(5.42)}$& 	CMTL$^{(6.33)}$& 	Whom$^{(3.25)}$& 	FlexClus$^{(4.33)}$&	Coclus$^{(4)}$&	MetaG$^{(5)}$&	Our Method$^\textbf{(1.67)}$ \\
   \midrule								
  Robot &	$\text{84.51}^{(7)}$&	$\text{84.82}^{(6)}$&	$\text{84.15}^{(8)}$&	$\textbf{88.90}^{(1)}$&	$\text{88.34}^{(4)}$&	$\text{87.83}^{(5)}$&	$\text{88.77}^{(2)}$&	$\text{88.67}^{(3)}$\\
  Vehicle&	$\text{79.73}^{(8)}$&	$\text{80.38}^{(6)}$&	$\text{80.23}^{(7)}$&	$\text{83.14}^{(4)}$&	$\text{82.45}^{(5)}$&	$\textbf{86.79}^{(1)}$&	$\text{83.53}^{(3)}$&	$\text{84.51}^{(2)}$\\
  Image&	$\text{97.08}^{(7)}$&	$\text{97.43}^{(3)}$&	$\text{97.09}^{(6)}$&	$\text{97.27}^{(4)}$&	$\text{98.05}^{(2)}$&	$\text{97.24}^{(5)}$&	$\text{97.05}^{(8)}$&	$\textbf{98.19}^{(1)}$\\
  Pen&	$\text{98.16}^{(7)}$&	$\text{98.28}^{(5.5)}$&	$\text{95.78}^{(8)}$&	$\text{98.28}^{(5.5)}$&	$\text{98.67}^{(3)}$&	$\textbf{99.26}^{(1)}$&	$\text{98.57}^{(4)}$&	$\text{99.12}^{(2)}$\\
  MNIST&	$\text{94.09}^{(7)}$&	$\text{94.87}^{(4)}$&	$\text{94.49}^{(6)}$&	$\text{95.56}^{(3)}$&	$\text{94.59}^{(5)}$&	$\text{93.09}^{(8)}$&	$\text{96.13}^{(2)}$&	$\textbf{96.70}^{(1)}$\\
 Letter&	$\text{84.12}^{(6)}$&	$\text{83.12}^{(8)}$&	$\text{85.62}^{(3)}$&	$\text{86.82}^{(2)}$&	$\text{83.72}^{(7)}$&	$\text{85.46}^{(4)}$&	$\text{85.41}^{(5)}$&	$\textbf{87.41}^{(1)}$\\
 \midrule								
 								
   20\%&	STL$^{(6)}$& 	MTL$^{(4.43)}$& 	CMTL$^{(6.14)}$& 	Whom$^{(3.29)}$& 	FlexClus$^{(5.57)}$&	Coclus$^{(4.57)}$&	MetaG$^{(4.71)}$&	Our Method$^\textbf{(1.14)}$ \\
   \midrule								
  Robot&	$\text{87.67}^{(7)}$&	$\text{88.23}^{(6)}$&	$\text{85.08}^{(8)}$&	$\textbf{90.76}^{(1)}$&	$\text{90.15}^{(3)}$&	$\text{88.43}^{(5)}$&	$\text{89.12}^{(4)}$&	$\text{90.34}^{(2)}$\\
  Vehicle&	$\text{85.88}^{(4)}$&	$\text{86.16}^{(3)}$&	$\text{82.29}^{(8)}$&	$\text{85.67}^{(6)}$&	$\text{85.29}^{(7)}$&	$\text{87.15}^{(2)}$&	$\text{85.78}^{(5)}$&	$\textbf{87.76}^{(1)}$\\
  Image&	$\text{97.41}^{(6)}$&	$\text{98.02}^{(3)}$&	$\text{97.32}^{(7)}$&	$\text{98.46}^{(2)}$&	$\text{97.44}^{(5)}$&	$\text{97.50}^{(4)}$&	$\text{97.29}^{(8)}$&	$\textbf{98.54}^{(1)}$\\
  Pen&	$\text{98.57}^{(7)}$&	$\text{99.01}^{(6)}$&	$\text{96.06}^{(8)}$&	$\text{99.14}^{(3)}$&	$\text{99.13}^{(4)}$&	$\text{99.30}^{(2)}$&	$\text{99.02}^{(4)}$&	$\textbf{99.63}^{(1)}$\\
  MNIST&	$\text{96.13}^{(6)}$&	$\text{96.71}^{(4)}$&	$\text{96.56}^{(5)}$&	$\text{96.76}^{(3)}$&	$\text{95.04}^{(7)}$&	$\text{94.09}^{(8)}$&	$\text{96.84}^{(2)}$&	$\textbf{97.86}^{(1)}$\\
 Landmine&	$\text{58.76}^{(8)}$&	$\text{61.89}^{(7)}$&	$\text{65.28}^{(2)}$&	$\text{62.53}^{(5)}$&	$\text{62.46}^{(6)}$&	$\text{63.52}^{(3)}$&	$\text{62.59}^{(4)}$&	$\textbf{65.82}^{(1)}$\\
 Letter&	$\text{88.75}^{(4)}$&	$\text{89.98}^{(2)}$&	$\text{88.24}^{(5)}$&	$\text{88.88}^{(3)}$&	$\text{83.79}^{(7)}$&	$\text{82.26}^{(8)}$&	$\text{87.99}^{(6)}$&	$\textbf{90.72}^{(1)}$\\
 \midrule								
   								
   50\%&	STL$^{(5.64)}$& 	MTL$^{(3.85)}$& 	CMTL$^{(6.29)}$& 	Whom$^{(3.29)}$& 	FlexClus$^{(6.21)}$&	Coclus$^{(5.29)}$ &	MetaG$^{(4.42)}$&	Our Method$^\textbf{(1)}$ \\
   \midrule								
  Robot&	$\text{91.26}^{(5.5)}$&	$\text{91.49}^{(3)}$&	$\text{86.26}^{(8)}$&	$\text{91.70}^{(2)}$&	$\text{91.26}^{(5.5)}$&	$\text{89.04}^{(7)}$&	$\text{91.27}^{(4)}$&	$\textbf{92.41}^{(1)}$\\
  Vehicle&	$\text{88.33}^{(3)}$&	$\text{88.71}^{(2)}$&	$\text{83.91}^{(8)}$&	$\text{87.3}^{(5)}$&	$\text{86.72}^{(7)}$&	$\text{87.55}^{(4)}$&	$\text{86.81}^{(6)}$&	$\textbf{89.83}^{(1)}$\\
  Image&	$\text{98.40}^{(6)}$&	$\text{98.43}^{(5)}$&	$\text{97.56}^{(8)}$&	$\text{98.58}^{(2)}$&	$\text{98.04}^{(7)}$&	$\text{98.52}^{(3)}$&	$\text{98.49}^{(4)}$&	$\textbf{99.07}^{(1)}$\\
  Pen&	$\text{98.77}^{(7)}$&	$\text{99.23}^{(5)}$&	$\text{96.17}^{(8)}$&	$\text{99.32}^{(4)}$&	$\text{99.33}^{(3)}$&	$\text{99.34}^{(2)}$&	$\text{99.21}^{(6)}$&	$\textbf{99.77}^{(1)}$\\
  MNIST&	$\text{97.20}^{(6)}$&	$\text{97.37}^{(4)}$&	$\text{97.31}^{(5)}$&	$\text{97.78}^{(3)}$&	$\text{96.60}^{(7)}$&	$\text{95.87}^{(8)}$&	$\text{98.46}^{(2)}$&	$\textbf{98.64}^{(1)}$\\
 Landmine&	$\text{63.76}^{(8)}$&	$\text{64.98}^{(6)}$&	$\text{66.76}^{(2)}$&	$\text{65.57}^{(4)}$&	$\text{64.87}^{(7)}$&	$\text{65.15}^{(5)}$&	$\text{66.24}^{(3)}$&	$\textbf{67.15}^{(1)}$\\
 Letter&	$\text{91.18}^{(4)}$&	$\text{91.62}^{(2)}$&	$\text{90.97}^{(5)}$&	$\text{91.25}^{(3)}$&	$\text{86.47}^{(7)}$&	$\text{86.27}^{(8)}$&	$\text{90.66}^{(6)}$&	$\textbf{92.49}^{(1)}$\\
 
 \bottomrule
 \end{tabular}
 }
 \end{center}
 \end{table}

In \tref{multi-task}, we report the average classification accuracy over $20$ runs of randomly sampled training sets for each experiment. Note that we utilized the method proposed in \cite{demvsar2006statistical} for our statistical analysis. More specifically, Friedman's and Holm's post-hoc tests at significance level $\alpha = 0.05$ were employed to compare our proposed method with the other methods.

 As shown in \tref{multi-task}, for each data set, Friedman's test ranks the best performing model as first, the second best as second and so on. The superscript next to each value in \tref{multi-task} indicates the rank of the corresponding model on the relevant data set, while the superscript next to each model reflects its average rank over all data sets for the corresponding training set size. 
Note that methods depicted in boldface are deemed statistically similar to our model, since their corresponding $p$-values are not smaller than the adjusted $\alpha$ values obtained by Holm's post-hoc test. Overall, it can be observed that our method dominates three, six and five out of seven methods, when trained with $10\%$, $20\%$ and $50\%$ training set sizes respectively.

\begin{table}[htpb] 
\begin{center}
 \caption{Comparison of our method against the other methods with the Holm test}
 \label{multi-task2}
  \tabcolsep=0.2cm
 \begin{tabular}{l c c c c c c c}
  \toprule
 
  10\%&	STL& 	MTL& 	CMTL& 	\textbf{Whom}& 	\textbf{FlexClus}&	\textbf{Coclus} &	\textbf{MeTaG}\\
   \midrule							
 Test statistic&	$\text{3.93}$&	$\text{2.13}$&	$\text{3.49}$&	$\text{1.25}$&	$\text{2.40}$&	$\text{2.62}$&	$\text{2.29}$\\
  p value&	$\text{0.0005}$&	$\text{0.0138}$&	$\text{0.0022}$&	$\textbf{0.2869}$&	$\textbf{0.0777}$&	$\textbf{0.1214}$&	$\textbf{0.1214}$\\
 Adjusted $\alpha$ &	$\text{0.0071}$&	$\text{0.0083}$&	$\text{0.0100}$&	$\text{0.0125}$&	$\text{0.01667}$&	$\text{0.0250}$&	$\text{0.0500}$\\
 \midrule	
 	
20\%&	STL& 	MTL& 	CMTL& 	\textbf{Whom}& 	\text{FlexClus}&	\text{Coclus} &	\text{MeTaG}\\
\midrule							
Test statistic&	$\text{3.71}$&	$\text{2.51}$&	$\text{3.82}$&	$\text{1.64}$&	$\text{3.38}$&	$\text{2.62}$&	$\text{2.73}$\\
p value&	$\text{0.00021}$&	$\text{0.0121}$&	$\text{0.0001}$&	$\textbf{0.1017}$&	$\text{0.0007}$&	$\text{0.0088}$&	$\text{0.0064}$\\
Adjusted $\alpha$ &	$\text{0.0083}$&	$\text{0.0250}$&	$\text{0.0071}$&	$\text{0.0500}$&	$\text{0.0100}$&	$\text{0.01667}$&	$\text{0.0125}$\\
\midrule							
 								
50\%&	STL& 	\textbf{MTL}& 	CMTL& 	\textbf{Whom}& 	\text{FlexClus}&	\text{Coclus} &	\text{MeTaG}\\
 \midrule							
 Test statistic&	$\text{3.55}$&	$\text{2.18}$&	$\text{4.04}$&	$\text{1.75}$&	$\text{3.98}$&	$\text{3.27}$&	$\text{2.61}$\\
 p value&	$\text{0.0004}$&	$\textbf{0.0291}$&	$\text{0.0001}$&	$\textbf{0.0809}$&	$\text{0.0001}$&	$\text{0.0011}$&	$\text{0.0089}$\\
 Adjusted $\alpha$ &	$\text{0.0100}$&	$\text{0.0250}$&	$\text{0.0071}$&	$\text{0.0500}$&	$\text{0.0083}$&	$\text{0.0125}$&	$\text{0.01667}$\\

 \bottomrule
 \end{tabular}
 \end{center}
 \end{table}

Also, in \fref{fig:Feature space}, we provide better insight of how the grouping of task feature spaces might be determined in our framework. For the purpose of visualization, we applied two Gaussian kernel functions with spread parameters $2$ and $2^8$ and used the \textit{Letter} multi-task data set.

In this figure, the $x$ and $y$ axes represent the weights of these two kernel functions for each task.  From \fref{fig:Feature space} \subref{fig:subfig1}, when a small training size ($10\%$) is chosen, it can be seen that our framework yields a cluster of $3$ tasks, namely \{``A" vs ``G", ``A" vs ``O", ``G" vs ``Y"\} that share a common feature space to benefit from each other's data. However, as the number $n$ of training samples per task increases, every task is allowed to employ its own feature space to guarantee good performance. This is shown in \fref{fig:Feature space} \subref{fig:subfig2}, which displays the results obtained for a $50\%$ training set size. Note, that the displayed \ac{MKL} weights lie on the $\theta_1 + \theta_2 = 1$ line due to the framework's $L_1$ \ac{MKL} weight constraint.

\begin{figure}[htbp]
\centering
\subfigure[Traning set size $10\%$]{
   \includegraphics[width=0.48\textwidth]{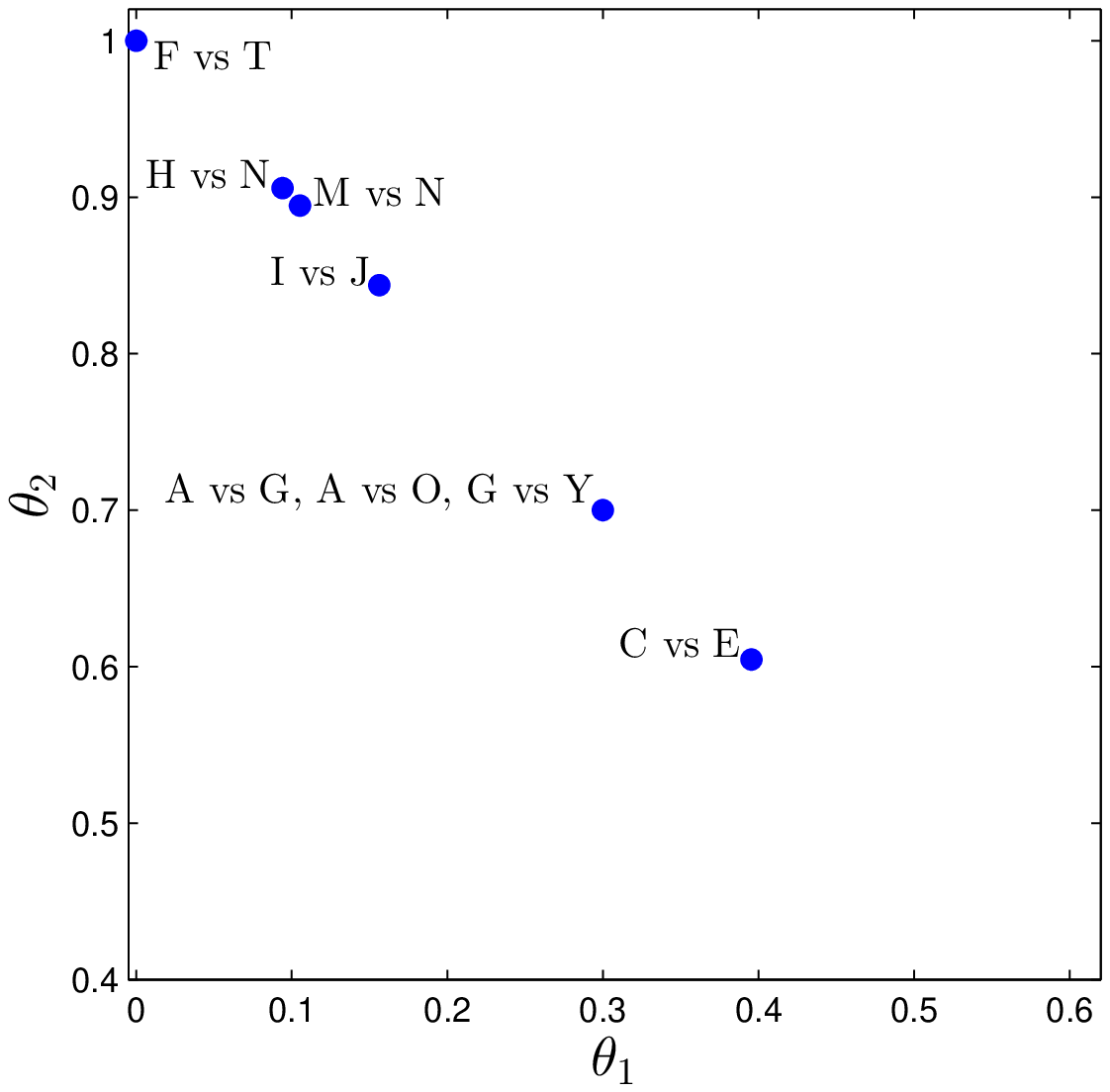}
    \label{fig:subfig1}
}
\subfigure[Traning set size $50\%$]{
    \includegraphics[width=0.48\textwidth]{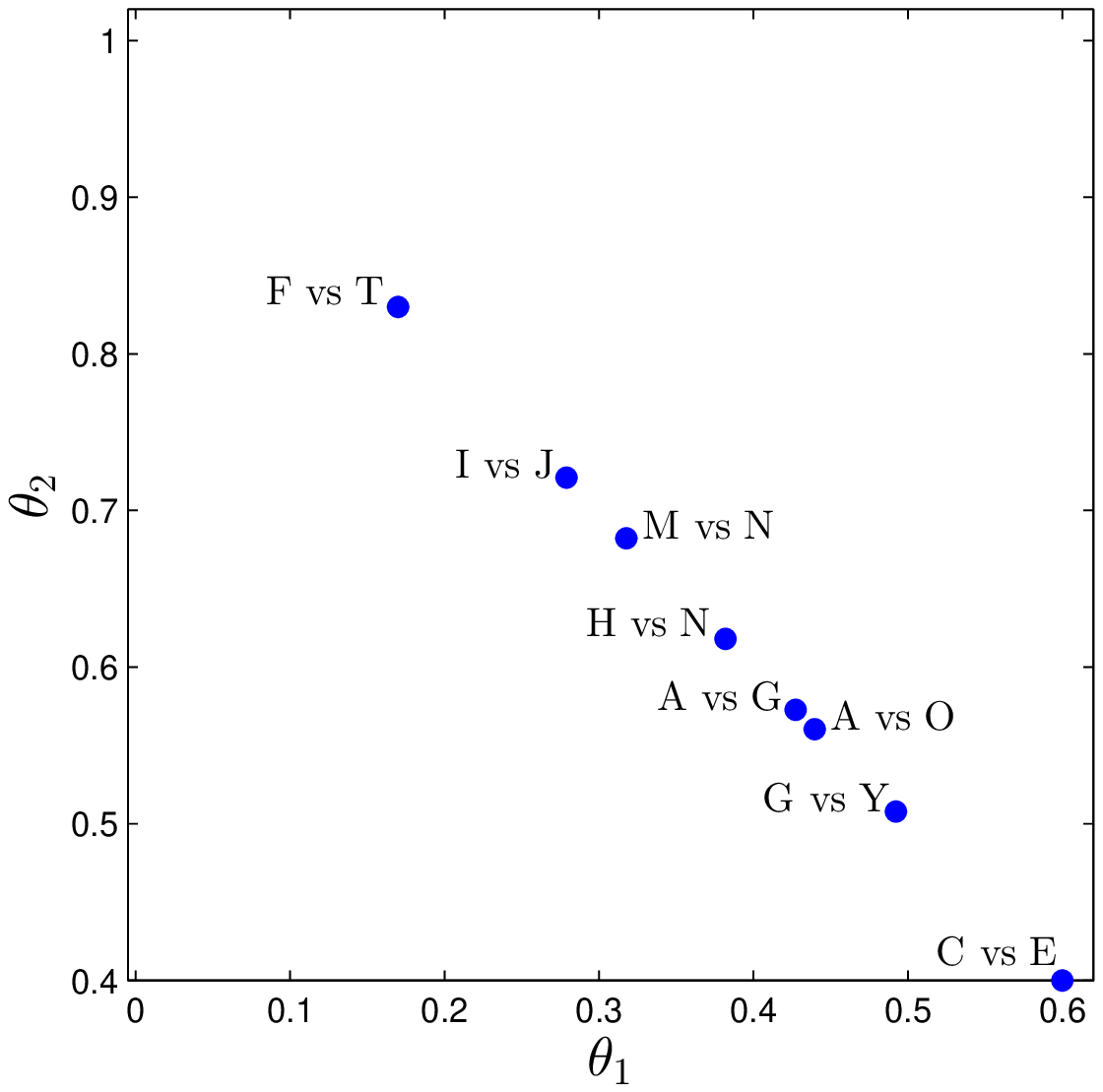}
    \label{fig:subfig2}
}

\caption{Feature space parameters for \textit{Letter} multi-task data set }
\label{fig:Feature space}
\end{figure}



%% file: Conclusions.tex
\section{Conclusions}
\label{sec:Conclusions}

In this work, we proposed a novel \ac{MT-MKL} framework for \ac{SVM}-based binary classification, where a flexible group structure is determined between each pair of tasks. In this framework, tasks are allowed to have a common, similar, or distinct feature spaces. Recently, some \ac {MTL} frameworks have been proposed, which also consider clustering strategies to capture task relatedness. However, our method is capable of modeling a more general type of task relationship, where tasks may be implicitly grouped according to a notion of feature space similarity. Also, our proposed optimization algorithm allows for a distributed implementation, which can be significantly advantageous for \ac{MTL} settings involving large number of tasks. The performance advantages reported on $7$ multi-task \ac{SVM}-based classification problems largely seem to justify our arguments in favor of our framework.

%% file: Acknowledgments.tex
\section*{Acknowledgments} 

N. Yousefi acknowledges support from National Science Foundation (NSF) grants No. 0806931 and No. 1161228. Moreover, M. Georgiopoulos acknowledges partial support from NSF grants No. 0806931, No. 0963146, No. 1200566, No. 1161228, and No. 1356233. Finally, G. C. Anagnostopoulos acknowledges partial support from NSF grant No. 1263011. Any opinions, findings, and conclusions or recommendations expressed in this material are those of the authors and do not necessarily reflect the views of the NSF.

%% file: Appendix.tex

\appendix
\newpage
\section*{Supplementary Materials}

A useful lemmas in deriving the generalization bound of \thrmref{MainTheorem} is provided next.




\begin{lemm}
  \label{lemm:lemma2}
  Let $\mathbf{A}, \mathbf{B} \in \mathbb{R}^{N \times N}$ and let $\boldsymbol{\sigma} \in \mathbb{R}^N$ be a vector of independent Rademacher random variables. Let $\circ$ denote the Hadamard (component-wise) matrix product. Then, it holds that
  
  \begin{align}
  	\mathbb{E}_{\boldsymbol{\sigma}} \left\{ \left( \boldsymbol{\sigma}' \mathbf{A} \boldsymbol{\sigma} \right)    \left( \boldsymbol{\sigma}' \mathbf{B} \boldsymbol{\sigma} \right) \right\}  = \trace{\mathbf{A}}  \trace{\mathbf{B}}  + 2 \left( \trace{\mathbf{A}\mathbf{B}} - \trace{\mathbf{A} \circ \mathbf{B}}  \right)
  \end{align}

  \end{lemm}
  \begin{proof}
  
  Let $[\cdot]$ denote the Iverson bracket, such that $[\mathrm{predicate}] = 1$, if  $\mathrm{predicate}$ is true and $0$, if false. The expectation in question can be written as

  \begin{align}
  	\label{eq:app1}
  	& \Expect[\boldsymbol{\sigma}]{  \left( \boldsymbol{\sigma}' \mathbf{A} \boldsymbol{\sigma} \right)    \left( \boldsymbol{\sigma}' \mathbf{B} \boldsymbol{\sigma} \right) } =  \sum_{i,j,k,l} a_{i,j} b_{k,l} \Expect{ \sigma_i \sigma_j \sigma_k \sigma_l}
  \end{align}

  where the indices of the last sum run over the set $\left\{1, \ldots, N\right\}$. Since the components of $\boldsymbol{\sigma}$ are independent Rademacher random variables, it is not difficult to verify the fact that $\Expect{ \sigma_i \sigma_j \sigma_k \sigma_l} = 1$ only in the following four cases: $\left\{i=k, j=l, i \neq l \right\}$, $\left\{ i=j, k=l, i \neq k \right\}$, $\left\{ i=l, k=j, i \neq k \right\}$ and $\left\{ i=j, j=k, k=l \right\}$; in all other cases, $\Expect{ \sigma_i \sigma_j \sigma_k \sigma_l} = 0$. Therefore, it holds that
  
  
  \begin{equation} \label{eq:app2}
  \begin{split}
  	 \Expect{ \sigma_i \sigma_j \sigma_k \sigma_l} & = [i=k][j=l][i \neq l] + [i=j][k=l][i \neq k] \\
  	 & + [i=l][k=j][i \neq k] +  [i=j][j=k][k=l]
  \end{split}
  \end{equation}

  Substituting \eref{eq:app2} into \eref{eq:app1}, after some algebraic operations, yields the desired result.
  
  \end{proof}

\section*{Proof of \thrmref{MainTheorem}}
\label{sec:ProofTheorem}
By utilizing Theorems 16, 17 in \cite{maurer2006bounds}, it can be proved that given a multi-task \ac {HS} $\mathcal{F}$, defines as a class of functions $\boldsymbol{f}=(f_1,\ldots,f_T) : \mathcal{X}\mapsto \mathbb{R}^T$, for all $f \in \mathcal{F}$, for $\delta > 0$ and for fixed $\rho > 0$, with probability at least $1-\delta$ the following holds
\noindent
\begin{align}
\label{eq:REC-Bound-Multi-task}
 R(f) \leq \hat{R}_{\rho}(f) + \frac{2}{\rho} \hat{\mathfrak{R}}_S (\mathcal{F}) + 3 \sqrt{\dfrac{\log\frac{2}{\delta}}{2Tn}}
\end{align}
\noindent
where the \ac {ERC} $\hat{\mathfrak{R}}_S (\mathcal{F})$ is given as 
\begin{align}
 \label{eq:RademacherMulti-task}
   \hat{\mathfrak{R}}_{S} (\mathcal{F}) = \frac{1}{nT} \mathrm{E}_{\sigma} \left\lbrace  \sup_{\boldsymbol{f}=(f_1,\ldots,f_T) \in \mathcal{F}} \sum_{t=1}^{T} \sum_{i=1}^{n} \sigma_{t}^{i} f_{t}(x_{t}^{i}) \right\rbrace 
 \end{align}
\noindent
and the $\rho$-empirical large margin error $\hat{R}_{\rho} (f)$ for the training sample $S = \left\{ \left( x_t^i, y_t^i \right) \right\}_{i,t=1}^{n,T}$ is defined as
\noindent
\begin{align}
\hat{R}_{\rho} (f)= \frac{1}{nT} \sum_{t=1}^{T} \sum_{i=1}^{n} \min \left( 1, [1-y_{t}^{i} f_{t}(x_{t}^{i})/\rho]_+   \right) 
\nonumber
\end{align}

Also, from eqs. (1) and (2) in \cite{cortes2010generalization}, we know that $\boldsymbol{w}_t = \sum_{j=1}^{n} \alpha_{t}^{j} \phi_{t}(x_{t}^{j})$ along with constraint $ \| \boldsymbol{w}_t \|^2 \leq R_{t}$, is equivalent to $\boldsymbol{\alpha}_t^{'} \boldsymbol{K}_t \boldsymbol{\alpha}_t \leq R_{t}$. Then we can observe that $\forall x \in S$ and $t \in{1,\ldots,T}$, the decision function defined as $f_{t}(x_t) = \langle \boldsymbol{w}_t ,\boldsymbol{\phi}_t(x_t) \rangle_{\mathcal{H}_{t,\boldsymbol{\theta}}}$ is equivalent to $f_{t}(x_t) = \sum_{j=1}^{n} \alpha_t^{j} K_{t} (x_t^{j},x_{t})$, where $\boldsymbol{K}_t = \sum_{m=1}^{M} \theta_{t}^{m} K_{m}$.

So, based on the definition of empirical Rademacher complexity given in \eqref{eq:RademacherMulti-task}, we will have
\begin{align}
 \label{eq:RademacherMulti-task-2}
   \hat{\mathfrak{R}}_{S} (\mathcal{H}) & =  \frac{1}{nT} \mathrm{E}_{\sigma} \left[ \sup_{ \boldsymbol{\theta}_t \in \varOmega^{''} ({\boldsymbol{\theta}}),\boldsymbol{\alpha}_t \in \varOmega ({\boldsymbol{\alpha}}) } \sum_{t=1}^{T} \sum_{i,j=1}^{n,n} \sigma_{t}^{i}\alpha_t^{j} {K}_{t} (x_t^{i},x_t^{j}) \right]
   \nonumber\\
     & = \frac{1}{nT} \mathrm{E}_{\sigma} \left[ \sup_{ \boldsymbol{\theta}_t \in \varOmega^{''} ({\boldsymbol{\theta}}), \boldsymbol{\alpha}_t \in \varOmega ({\boldsymbol{\alpha}}) } \sum_{t=1}^{T} \boldsymbol{\sigma}_{t}^{'} \boldsymbol{K}_{t} \boldsymbol{\alpha}_{t}  \right]\
 \end{align}
\noindent
where $\boldsymbol{\sigma}_{t}= [\sigma_t^1,\ldots,\sigma_t^{n}]^{'}$, $\boldsymbol{\alpha}_{t}= [\alpha_t^1,\ldots,\alpha_t^{n}]^{'}$, $\boldsymbol{K}_t \in \mathbb{R}^{n\times n}$ is a kernel matrix whose $(i,j)$-th elements is defined as $\sum_{m=1}^{M} \theta_{t}^{m} K_{m}(x_t^{i},x_t^{j})$, $\varOmega ({\boldsymbol{\alpha}})= \{ \boldsymbol{\alpha}_t \mid \boldsymbol{\alpha}_t^{'} \boldsymbol{K}_t \boldsymbol{\alpha}_t \leq R_{t},\; \forall t \}  $ and $\varOmega^{''} ({\boldsymbol{\theta}})$ is defined as \eqref{ThetaSpace3}.

It can be observed that the maximization problem with respect to $\boldsymbol{\alpha}_t$ can be handled as $T$ independent optimization problem, as $\varOmega ({\boldsymbol{\alpha}})$ is separable in terms of $\boldsymbol{\alpha}_t$. Also, it can be shown that using Cauchy-Schwartz inequality, the optimal value of $\boldsymbol{\alpha}_t$ is achieved when $\boldsymbol{K}_t^{1/2} \boldsymbol{\alpha}_t$ is colinear with $\boldsymbol{K}_t^{1/2} \boldsymbol{\sigma}_t$, which gives 
\begin{align}
\sup_{\boldsymbol{\alpha}_t \in \varOmega(\boldsymbol{\alpha})} \boldsymbol{\sigma}_t^{'} \boldsymbol{K}_t \boldsymbol{\alpha}_t = \sqrt {\boldsymbol{\sigma}_t^{'} \boldsymbol{K}_t \boldsymbol{\sigma}_t  R_{t}}
\nonumber
\end{align}

Assuming $R_{t} \leq R\; \forall t$, \eqref{eq:RademacherMulti-task-2} now becomes
\begin{align}
  \label{eq:Rademacher-Multi-task3.3}
   \hat{\mathfrak{R}}_{S} (\mathcal{H}) & = \frac{\sqrt{R}}{nT} \mathrm{E}_{\sigma} \left\lbrace  \sup_{ \boldsymbol{\theta}_t \in \varOmega^{''} ({\boldsymbol{\theta}}) } \sum_{t=1}^{T} \sqrt{ \boldsymbol{\sigma}_{t}^{'} \boldsymbol{K}_{t} \boldsymbol{\sigma}_{t} } \right\rbrace \
   \nonumber\\
   & = \frac{\sqrt{R}}{nT} \mathrm{E}_{\sigma} \left\lbrace  \sup_{ \boldsymbol{\theta}_t \in \varOmega^{''} ({\boldsymbol{\theta}}) } \sum_{t=1}^{T} \sqrt{ \sum_{m=1}^{M} \theta_{t}^{m} (\boldsymbol{\sigma}_{t}^{'} \boldsymbol{K}_{t}^{m} \boldsymbol{\sigma}_{t}) } \right\rbrace \
   \nonumber\\
   & = \frac{\sqrt{R}}{nT} \mathrm{E}_{\sigma} \left\lbrace  \sup_{ \boldsymbol{\theta}_t \in \varOmega^{''} ({\boldsymbol{\theta}}) } \sum_{t=1}^{T} \sqrt{ \boldsymbol{\theta}_{t}^{'} \boldsymbol{u}_{t} } \right\rbrace \
 \end{align}
\noindent
where $\boldsymbol{\theta}_{t} = [\theta_t^1,\ldots,\theta_t^{M}]^{'}$, $\boldsymbol{u}_{t} = [u_t^1,\ldots,u_t^{M}]^{'}$, and $u_t^{m}= \boldsymbol{\sigma}_{t}^{'} \boldsymbol{K}_{t}^{m} \boldsymbol{\sigma}_{t}$ . Note that \eqref{eq:Rademacher-Multi-task3.3} can also be upper-bounded. In particular, assuming $\omega_{t}= \sqrt{\boldsymbol{\theta}_{t}^{'} \boldsymbol{u}_{t} }$, and using the H\"{o}lder's inequality $\| \boldsymbol{x} \boldsymbol{y}\|_r \leq \| \boldsymbol{x}\|_p  \| \boldsymbol{y}\|_q$, for  $p=2$, $r=1$ and $\boldsymbol{y}=\mathbf{1}_n$, we will have 
\begin{align}
\sum_{t=1}^{T} \sqrt{ \boldsymbol{\theta}_{t}^{'} \boldsymbol{u}_{t} } =\left( \sum_{t=1}^{T} w_{t} \right)  = \| \boldsymbol{\omega} \|_1 \leq \sqrt{T} \| \boldsymbol{\omega} \|_2 = \sqrt{T} \left( \sum_{t=1}^{T} (w_{t})^{2} \right)^{1/2} = \sqrt{T} \sqrt {\sum_{t=1}^{T} \boldsymbol{\theta}_{t}^{'} \boldsymbol{u}_{t} }
\nonumber
\end{align}

Therefore, we can upper bound the Rademacher complexity \eqref{eq:Rademacher-Multi-task3.3} as follows 
\begin{align}
  \label{eq:Rademacher-Multi-task-upper-bounded3}
   \hat{\mathfrak{R}}_{S} (\mathcal{H}) & = \frac{\sqrt{R}}{nT} \mathrm{E}_{\sigma} \left\lbrace  \sup_{ \boldsymbol{\theta}_t \in \varOmega^{''} ({\boldsymbol{\theta}}) } \sum_{t=1}^{T} \sqrt{ \boldsymbol{\theta}_{t}^{'} \boldsymbol{u}_{t} } \right\rbrace \
   \nonumber\\
   & \leq 
   \frac {1}{n} \sqrt{ \frac {R} {T} } \; \mathrm{E}_{\sigma} \left\lbrace  \sup_{ \boldsymbol{\theta}_t \in \varOmega^{''} ({\boldsymbol{\theta}}) }  \sqrt{ \sum_{t=1}^{T} \boldsymbol{\theta}_{t}^{'} \boldsymbol{u}_{t} } \right\rbrace \
   \nonumber\\
   & = 
   \frac {1}{n} \sqrt{ \frac {R} {T} } \; \mathrm{E}_{\sigma} \left\lbrace  \sup_{ \boldsymbol{\theta}_t \in \varOmega^{''} ({\boldsymbol{\theta}}) }  \sqrt{ \trace { \mathbf{\Theta}^{'} \mathbf{U} }  }  \right\rbrace \
 \end{align}
 \noindent
 where $\mathbf{\Theta}= [\boldsymbol{\theta}_1,\ldots,\boldsymbol{\theta}_T] \in \mathbb{R}^{M \times T}$ and $\mathbf{U}= [\boldsymbol{u}_1,\ldots,\boldsymbol{u}_T] \in \mathbb{R}^{M \times T}$. Also, by contradiction, it can be easily proved that 
 \begin{align}
  \arg \max_{\mathbf{\Theta}} \sqrt{\trace { \mathbf{\Theta} ^{'} \mathbf{U} }} = \arg \max_{\mathbf{\Theta}} \left\lbrace  \trace{ \mathbf{\Theta} ^{'} \mathbf{U} }\right\rbrace 
  \nonumber
 \end{align}
 
 Using the Lagrangian multiplier method, the optimization w.r.t. $\mathbf{\Theta}$ yields the optimal value for $\mathbf{\Theta}$ as 
 \begin{align}
  \mathbf{\Theta}^{*}= \frac{1}{2 \alpha  T }    \mathbf{U} \mathbf{P}_T
  \nonumber
 \end{align}
 \noindent
 where $\mathbf{P}_T \in \mathbb{R}^{T \times T}$ is a centering matrix as we defined in \sref{sec:Optimization Algorithm}. Moreover, $ \alpha = (1/2\gamma)\sqrt{a-(1/T)b}$, $ a = \trace{\mathbf{U} \mathbf{U}^{'}}$, and $b = \trace{ \mathbf{U} \mathbf{1}_{T} {\mathbf{1}_{T}}^{'} \mathbf{U}^{'}}$.

 substituting the optimal value of $\mathbf{\Theta}$ in \eqref{eq:Rademacher-Multi-task-upper-bounded3}, finally yields
 \begin{align}
    \hat{\mathfrak{R}}_{S} (\mathcal{H}) & \leq
    \frac {\sqrt{\gamma R}} {n T^{3/4}}  \; \mathrm{E}_{\sigma} \left\lbrace  \sqrt{ a- (1/T) b}  \right\rbrace ^{1/2} \ 
    \nonumber
  \end{align}
  By applying Jensen's inequality twice, we obtain
  \begin{align}
     \label{eq:Rademacher-Multi-task-upper-bounded5}
      \hat{\mathfrak{R}}_{S} (\mathcal{H}) & \leq   \frac {\sqrt{\gamma R}} {n T^{3/4}} \; \left\lbrace    \sqrt{ \mathrm{E}_{\sigma} (a-(1/T)b)}  \right\rbrace ^{1/2} \
    \end{align}
  
  From the definition, we can see that both $a$ and $b$ depend on variable $\sigma$. If we define $\boldsymbol{u}^{m}=[\boldsymbol{\sigma}_1^{'} \boldsymbol{K}_{1}^{m}\boldsymbol{\sigma}_1,\ldots,\boldsymbol{\sigma}_T^{'} \boldsymbol{K}_{T}^{m}\boldsymbol{\sigma}_T ]^{'}$ as the row vector of matrix $\mathbf{U}$, and $\boldsymbol{d} \triangleq \left[ \trace{\sum_{m=1}^M \boldsymbol{K}_1^m} \ldots \trace{\sum_{m=1}^M \boldsymbol{K}_T^m} \right]'$ then  with the help of Lemma \ref{lemm:lemma2}, it can be shown that


 \begin{align}
  \label{eq:Expectation-a^2,b}
  & \mathrm{E}_{\sigma} (a) =       \boldsymbol{d}^{'}   \boldsymbol{d} +2 \sum_{t=1}^{T} \sum_{m=1}^{M} \left[  \trace {\boldsymbol{K}_{t}^{m} \boldsymbol{K}_{t}^{m}} - \trace { \boldsymbol{K}_{t}^{m} \circ \boldsymbol{K}_{t}^{m} } \right] 
  \nonumber\\
  & \mathrm{E}_{\sigma} (b)  =    \boldsymbol{d}^{'} \mathbf{1}_{T}  {\mathbf{1}_{T}}^{'} \boldsymbol{d} + 2 \sum_{t=1}^{T} \sum_{m=1}^{M} \left[ \trace {\boldsymbol{K}_{t}^{m} \boldsymbol{K}_{t}^{m}} - \trace { \boldsymbol{K}_{t}^{m} \circ \boldsymbol{K}_{t}^{m} } \right]  
 \end{align}
\noindent
Considering the fact that $\trace { \boldsymbol{K}_{t}^{m} \circ \boldsymbol{K}_{t}^{n}} \geq 0 $, and $\trace {\boldsymbol{K}_{t}^{m} \boldsymbol{K}_{t}^{n}} \geq 0\; \forall t,m,n$, and assuming that $\boldsymbol{K}_{t}^{m} (x,x) \leq 1 \; \forall x,t,m$, it can be shown that

\begin{align}
\label{eq:Expectation-a^2,b-2}
  \mathrm{E}_{\sigma} (a) - \frac{1}{T} \mathrm{E}_{\sigma} (b) \leq T M^{2} n^{2} \left\lbrace  1 + \frac{2}{M} + \frac{2}{TMn} \right\rbrace \leq 3 T M^{2} n^{2}  
 \end{align} 
 \noindent
Combining \eqref {eq:Rademacher-Multi-task-upper-bounded5}, and \eqref{eq:Expectation-a^2,b-2} and after some algebra operations,  we conclude that, if 

\begin{align}
\label{Rademacher-Upperbound-final-2}
\hat{R}_{ub} \left( \mathcal{H} \right) \triangleq  \sqrt{ \frac { \sqrt{3}\gamma R M} {n T} } 
\end{align}
\noindent
then $\hat{R}(\mathcal{H}) \leq \hat{R}_{ub}\left( \mathcal{H} \right)$. This last fact in conjunction with \eref{eq:REC-Bound-Multi-task} conclude the theorem's statement.
